\documentclass[conference]{IEEEtran}
\IEEEoverridecommandlockouts

\usepackage{amsmath,amssymb,amsfonts,amsthm,mathtools}
\usepackage{graphicx}
\usepackage{subfigure}
\usepackage{booktabs}
\usepackage{multirow}
\usepackage{float}
\usepackage{algorithm}
\usepackage{algorithmic}
\usepackage{microtype}
\usepackage{textcomp}
\usepackage{xcolor}
\usepackage{listings}
\usepackage{tikz}
\usetikzlibrary{calc, arrows.meta, positioning}
\usepackage{hyperref}
\hypersetup{hidelinks}
\usepackage[capitalize,noabbrev]{cleveref}
\usepackage{cite}

\tikzset{
  procstep/.style={
    draw, thick, rounded corners,
    minimum width=4cm,
    minimum height=1.5cm,
    align=center,
    font=\small,
    top color=cyan!15,
    bottom color=cyan!30
  },
  repcsarrow/.style={
    ->, very thick, shorten >=2pt,
    >=Latex, color=black!70
  }
}

\DeclareMathOperator{\KL}{KL}
\usepackage{enumitem}
\DeclareMathOperator{\Ent}{H}
\newcommand{\Prag}{P_{\text{rag}}}
\newcommand{\Ppara}{P_{\text{para}}}
\newcommand{\E}{\mathbb{E}}

\DeclareUnicodeCharacter{2248}{\approx} 

\theoremstyle{plain}
\newtheorem{theorem}{Theorem}[section]
\newtheorem{lemma}[theorem]{Lemma}

\theoremstyle{definition}
\newtheorem{definition}[theorem]{Definition}
\newtheorem{assumption}{Assumption}
\newtheorem{corollary}{Corollary}

\title{RePCS: Diagnosing Data Memorization in LLM-Powered Retrieval-Augmented Generation}

\author{
\IEEEauthorblockN{Le Vu Anh\IEEEauthorrefmark{1}, Nguyen Viet Anh\IEEEauthorrefmark{2}, Mehmet Dik\IEEEauthorrefmark{3}, Luong Van Nghia\IEEEauthorrefmark{4}}
\IEEEauthorblockA{\IEEEauthorrefmark{1}Institute of Information Technology, Vietnam Academy of Science and Technology, Hanoi, Vietnam \\
Email: anhlv@ioit.ac.vn}
\IEEEauthorblockA{\IEEEauthorrefmark{2}Institute of Information Technology, Vietnam Academy of Science and Technology, Hanoi, Vietnam \\
Email: anhnv@ioit.ac.vn}
\IEEEauthorblockA{\IEEEauthorrefmark{3}Department of Mathematics, Computer Science \& Physics, Rockford University, Illinois, United States \\
Email: mdik@rockford.edu}
\IEEEauthorblockA{\IEEEauthorrefmark{4}Department of Information Technology, Dong A University, Da Nang, Vietnam \\
Email: nghialv@donga.edu.vn}
}

\begin{document}
\maketitle

\begin{abstract}
Retrieval-augmented generation (RAG) has become a common strategy for updating large language model (LLM) responses with current, external information. However, models may still rely on memorized training data, bypass the retrieved evidence, and produce contaminated outputs. We introduce \emph{Retrieval-Path Contamination Scoring} (RePCS), a diagnostic method that detects such behavior without requiring model access or retraining. RePCS compares two inference paths: (i) a \emph{parametric} path using only the query, and (ii) a \emph{retrieval-augmented} path using both the query and retrieved context by computing the Kullback--Leibler (KL) divergence between their output distributions. A low divergence suggests that the retrieved context had minimal impact, indicating potential memorization. This procedure is model-agnostic, requires no gradient or internal state access, and adds only a single additional forward pass. We further derive PAC-style guarantees that link the KL threshold to user-defined false positive and false negative rates. On the \textsc{Prompt-WNQA} benchmark, RePCS achieves a ROC-AUC of 0.918. This result outperforms the strongest prior method by 6.5 percentage points while keeping latency overhead below 4.7\% on an NVIDIA T4 GPU. RePCS offers a lightweight, black-box safeguard to verify whether a RAG system meaningfully leverages retrieval, making it especially valuable in safety-critical applications.
\end{abstract}


\section{Introduction}

Over the past five years, \textit{large-scale language models} (LLMs) such as GPT-4, Claude 3, and Llama 2 have transformed human–computer interaction: users can pose open-ended questions and receive coherent, context-aware answers in natural language \cite{10.1145/3735632}. To increase factual reliability, industry and academia have turned to the \textit{Retrieval-Augmented Generation} (RAG) paradigm, in which an LLM is paired with a search layer that retrieves passages from an external knowledge base (KB) \cite{li-etal-2024-know}. By explicitly searching for supporting evidence, RAG promises to reduce hallucinations, provide up-to-date information beyond the model’s pre-training cut-off, and offer verifiable citations, making it the de-facto architecture in production chatbots, search engines, and enterprise assistants \cite{ayala-bechard-2024-reducing}.

In modern wireless and networked systems, RAG is increasingly used to fetch live network state data such as channel quality measurements, interference statistics, or handover logs before making critical decisions on power control, resource allocation, and mobility management \cite{10620276, 10.1145/3711992.3711996}. However, we observe a problematic failure mode in this adoption: the LLM models embedded within RAG systems may silently skip fetching fresh telemetry and instead replay information memorized during its pre-training. This silent data contamination in network state queries causes controllers to act on stale or irrelevant facts without any indication of error, potentially destabilizing link budgets or misallocating spectrum.

Early RAG systems relied on lexical retrievers such as BM25 \cite{wang-etal-2024-searching}, but the field quickly moved toward \emph{dense dual-encoders} \cite{10.1145/3626772.3657968} and late-interaction models \cite{santhanam-etal-2022-colbertv2}. Alongside retriever improvements, a second line of work fine-tunes LLMs to better exploit retrieved passages, using techniques ranging from contrastive reward signals to retrieval-conditioned causal masking \cite{asai2024selfrag}. Parallel to these engineering advances, the evaluation community has sounded the alarm on data contamination, the hidden overlap between training corpora and supposedly “held-out” test sets \cite{yao2024datacontam}. In order to detect this issue, researchers have proposed matching n-grams, analysing log-probabilities, or tracing influence functions to detect when benchmark results are inflated by memorized content \cite{manakul2023selfcheckgpt, sun2025redeep, zimmerman2024twotiered}.

Despite gains in factuality and benchmarking rigour, prevailing methods share a critical assumption. If the retriever returns high-quality passages, the generator will faithfully incorporate them. In practice, a powerful LLM may already remember the answer from pre-training, quietly ignoring or only superficially citing the external context. Under such circumstances, even a perfectly engineered retriever cannot prevent leakage; evaluation metrics can be silently inflated, and network operators may over-trust the system’s grounding.

Existing contamination detectors still leave two critical gaps. First, most are \emph{intrusive}: influence-function analysis, gradient masking, or logit inspection demands white-box access to full-precision parameters, which is often impossible for proprietary or quantized models. Second, they probe individual tokens or retrieved passages, but never ask a simpler question: \emph{Does the final answer change when retrieval is removed?} As a result, practitioners lack a lightweight, black-box test that verifies whether a RAG pipeline truly relies on the documents it fetches.

\textbf{Retrieval-Path Contamination Scoring (RePCS)} closes this gap. For every query we execute the \emph{same} LLM model along two inference paths:

\begin{enumerate}
  \item \textbf{Retrieval-augmented path}: the query combined with the top-$K$ passages returned by the search layer;
  \item \textbf{Parametric path}: the query alone, with no external context.
\end{enumerate}

We then compute a single Kullback–Leibler (KL) divergence between the two answer-probability distributions. A small KL value indicates that the retrieved context has little influence. This is evidence that the model is relying on memorised training data, whereas a large KL value certifies that retrieval introduces new information. The score is obtained post-hoc, without gradients, logits, or model modifications, and costs only one extra forward pass.

Our contributions include:
\begin{itemize}
    \item \textbf{Training-free, black-box detector}. RePCS compares a retrieval-augmented run with a parametric run of the same LLM; one additional forward pass and a KL divergence flag potential memorisation in real time.
    \item \textbf{Provable guarantees}. We derive PAC-style bounds that map the KL threshold to user-specified false-positive and false-negative rates, providing principled control over detection sensitivity.
    \item \textbf{Practical speed}. On the \textsc{Prompt-WNQA} benchmark, RePCS lifts ROC-AUC by 6.5\,pp over the strongest baseline while adding $\le 4.7\%$ latency on an NVIDIA T4 GPU, making it suitable for live deployments.
\end{itemize}

The rest of the paper is organized as follows. Section~\ref{sec:related} surveys related work in RAGs and data memorization detection. Section~\ref{sec:theory} formalizes RePCS, detailing the dual-path LLM protocol and KL scoring rule, and presents our theoretical guarantees. Section~\ref{sec:mainalgo} describes the algorithmic framework. Section~\ref{sec:experiments} outlines the experimental setup and empirical results. Finally, Section~\ref{sec:conclusion} summarizes key findings and discusses future directions.


\section{Related Work}\label{sec:related}

Early black-box approaches treat the language model as an oracle and rely on output variability to detect memorization.  \textsc{SelfCheckGPT} (2023) re-generates an answer several times and marks it as unsafe when the paraphrases disagree, showing that output instability is a useful clue even without access to model internals or the retriever \cite{manakul2023selfcheckgpt}.  While simple, the method needs several additional calls per user query, increasing latency and computation cost.

A second research line opens the model to gather richer signals.  \textsc{ReDeEP} (2025) traces token probabilities back to individual attention heads and reports that some heads consistently carry memorized content, whereas others depend on retrieved passages \cite{sun2025redeep}.  \textsc{LLM-Check} (2024) also measures hidden-state norms and attention-map entropy, offering both white- and grey-box variants that raise precision but still require layer-level features or gradient access \cite{sriramanan2024llmcheck}.  These probes achieve high accuracy yet face compatibility issues when models are quantised or internal layouts change.

To meet real-time budgets, industry deployments favour lightweight detectors trained with supervision.  The Two-Tiered system (2024) fine-tunes two compact encoders: one over the query plus retrieved passages and another over the final answer and flags hallucination when their embeddings diverge, adding only a few milliseconds to each request \cite{zimmerman2024twotiered}.  \textsc{Luna} (2025) pushes latency lower by distilling a DeBERTa-large model into a 45 MB checkpoint that keeps GPT-3.5-level detection accuracy with one-tenth the inference cost \cite{belyi2025luna}.  Both methods, however, depend on thousands of labelled examples and must be re-trained whenever the base LLM or retriever changes.

An orthogonal direction checks factual claims directly.  \textsc{RefChecker} (2024) converts an answer into subject–predicate–object triples and verifies each triple against the retrieved documents, catching span-level conflicts when evidence is missing or contradictory \cite{hu2024knowledge}.  Reliability here depends on the quality of open-IE extraction and full evidential coverage: if the retriever misses a supporting passage, the checker may wrongly label a true statement as hallucinated.  Progress in supervised detection and claim checking has been accelerated by \textsc{RAGTruth} (2024), a benchmark of 18 000 RAG answers with word-level hallucination labels that enables fair comparison across methods and fuels new detectors \cite{niu2024ragtruth}.

Several studies move beyond detection to repair or avoidance.  \textsc{RAG-HAT} (2024) feeds detector-identified hallucination spans back into the prompt of a large model and asks it to rewrite its own answer, reducing factual errors without extra human input \cite{song2024raghat}.  In a complementary path, counterfactual prompting method (2024) treats the detector’s score as a risk estimate and trains the generator to abstain when that risk exceeds a user-defined threshold, trading coverage for higher reliability in safety-critical settings \cite{chen2024risk}.

\textbf{Our observations}.  Existing detectors either require multiple language-model calls, demand intrusive access to gradients or attention weights, or depend on supervised training with large labelled sets. Each factor here adds cost, latency, or maintenance overhead.  Specifically, none of them tests whether the \emph{same} LLM, run once with retrieved passages and once without, yields answer distributions that differ enough to prove genuine grounding, nor do they provide user-tunable statistical guarantees on false-alarm rates.  These gaps motivate our \textsc{RePCS} framework, which stays black-box, needs no additional training, measures a single KL divergence between the retrieval-augmented and parametric answer distributions, and offers provable error bounds while adding less than five per-cent latency.

\section{Preliminaries} \label{sec:theory}

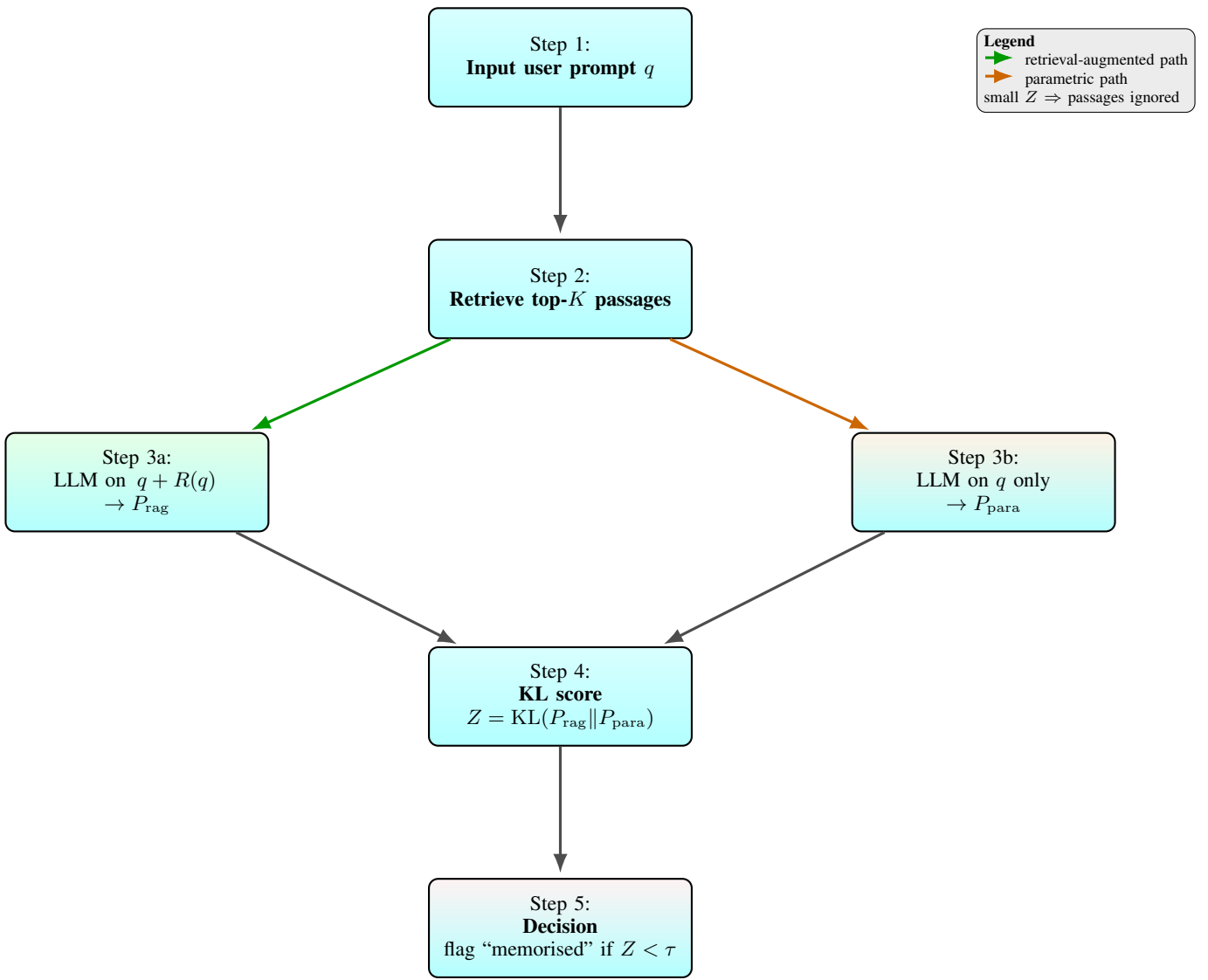
\begin{figure*}[t]
  \centering
  \rule{\textwidth}{0.4pt}\vspace{0.5em}

  \begin{tikzpicture}[node distance=2cm]
    \node[procstep]                                  (Q) {Step 1:\\\textbf{Input user prompt \(q\)}};
    \node[procstep, below=of Q]                      (R) {Step 2:\\\textbf{Retrieve top-\(K\) passages}};
    \node[procstep, below left=of R, xshift=-1cm,
          fill=green!20!white, top color=green!10]    (A)
        {Step 3a:\\LLM on \(\,q + R(q)\,\)\\\(\to P_{\mathrm{rag}}\)};
    \node[procstep, below right=of R, xshift=1cm,
          fill=orange!25!white, top color=orange!10]  (B)
        {Step 3b:\\LLM on \(q\) only\\\(\to P_{\mathrm{para}}\)};
    \node[procstep, below=of $(A)!0.5!(B)$, yshift=-0.5cm,
          fill=purple!15]                            (K)
        {Step 4:\\\textbf{KL score}\\\(Z=\KL(P_{\mathrm{rag}}\|P_{\mathrm{para}})\)};
    \node[procstep, below=of K,
          fill=red!15!white, top color=red!5]        (D)
        {Step 5:\\\textbf{Decision}\\flag “memorised” if \(Z<\tau\)};

    \draw[repcsarrow] (Q) -- (R);
    \draw[repcsarrow, color=green!60!black]  (R) -- (A);
    \draw[repcsarrow, color=orange!80!black] (R) -- (B);
    \draw[repcsarrow] (A) -- (K);
    \draw[repcsarrow] (B) -- (K);
    \draw[repcsarrow] (K) -- (D);

    \begin{scope}[shift={(8cm,-0.2cm)}]
      \node[draw, fill=gray!15, rounded corners,
            font=\scriptsize, inner sep=3pt, align=left] {
        \textbf{Legend}\\
        \tikz{\draw[repcsarrow, color=green!60!black] (0,0) -- (0.5,0);} retrieval-augmented path\\
        \tikz{\draw[repcsarrow, color=orange!80!black] (0,0) -- (0.5,0);} parametric path\\
        small \(Z\) \(\Rightarrow\) passages ignored
      };
    \end{scope}
  \end{tikzpicture}

  \vspace{0.6em}\rule{\textwidth}{0.4pt}
  \caption{Overview of \textsc{RePCS}. We run the LLM twice: once with retrieved passages and once without. We then use a single KL score to flag possible memorization.}
  \label{fig:repcs_flow_nice}
\end{figure*}

\subsection{Concept and Notation}
\label{subsec:intuition}

We begin by formalising the two evaluation pathways that are critical to \textsc{RePCS}.

Let $q$ be a user query drawn from natural language. A retriever module returns the $K$ most relevant passages, written
$R(q)=\{c_1,\dots,c_K\}$.  A \emph{frozen} large language model (LLM)~$\mathcal{M}$ can now be
queried in two distinct modes \cite{guu2020realm}:
\[
\begin{aligned}
  \text{(i) Retrieval-augmented path } &
      \mathcal{P}_{\mathrm{rag}} :
      \mathcal{M}\bigl(\langle q,R(q)\rangle\bigr)\;\longrightarrow\;P_{\mathrm{rag}},\\[2pt]
  \text{(ii) Parametric path } &
      \mathcal{P}_{\mathrm{para}} :
      \mathcal{M}(q)\;\longrightarrow\;P_{\mathrm{para}}.
\end{aligned}
\]

The outputs $P_{\mathrm{rag}},P_{\mathrm{para}}\in\Delta^{V\times T}$ are full token-probability tensors over a vocabulary of size~$V$ for the first~$T$ generated positions. Throughout the paper we treat both tensors as probability \textit{distributions} and make no further assumptions about the internal architecture of~$\mathcal{M}$.

In a RAG system, useful evidence in $R(q)$ \emph{ought} to perturb the LLM’s predictive distribution. If instead we observe
$P_{\mathrm{rag}}\approx P_{\mathrm{para}}$, the passages have had virtually no influence, and the answer most likely originates from the model’s parametric memory \cite{izacard2021eacl}. Detecting this “silent fallback’’ is exactly the goal of
\textsc{RePCS}.

Rather than inspecting two $V{\times}T$ arrays, we compress their
difference into the Kullback–Leibler divergence \cite{csiszar1967fdiv}
\[
  Z(q) \;=\;
  \KL\!\bigl(P_{\mathrm{rag}}\;\|\;P_{\mathrm{para}}\bigr),
\]
which carries several advantages:
\begin{itemize}[leftmargin=*]
  \item \emph{Interpretability}: $Z(q)$ quantifies, in nats, the extra
        log-likelihood required to pretend that retrieval mattered when
        it did not.
  \item \emph{Monotonicity}: $Z(q)$ equals zero \emph{iff}
        $P_{\mathrm{rag}}\!=\!P_{\mathrm{para}}$ and increases
        smoothly with any deviation.
  \item \emph{Statistical tractability}: KL enjoys tight
        concentration and minimax bounds, allowing finite-sample error guarantees \cite{hoeffding1963probability, massart1990dkw}.
\end{itemize}
A learned threshold~$\tau$ is finally applied:
queries with $Z(q)<\tau$ are flagged as \emph{memorised}.

For convenience, \Cref{tab:symbols} lists the primary symbols used
throughout the preliminaries section.

\begin{table}[t]
\centering
\caption{Primary symbols used in \S\ref{sec:theory}.}
\label{tab:symbols}
\begin{tabular}{ll}
\toprule
Symbol & Description \\ \midrule
$q$ & user query \\
$R(q)$ & retrieved evidence (top-$K$ passages) \\
$P_{\mathrm{rag}},P_{\mathrm{para}}$ & token distributions on the two paths \\
$Q$ & distribution conditioned on $R(q)$ alone \\
$\eta$ & retrieval-influence coefficient (Def.~\ref{def:eta}) \\
$Z(q)$ & KL score $\KL(P_{\mathrm{rag}}\|P_{\mathrm{para}})$ \\
$\tau$ & decision threshold (learned) \\
$T$ & inspected output length (tokens) \\
$\gamma$ & bound on per-token log-ratio \\
$\alpha,\epsilon$ & target false-positive / false-negative rates \\
\bottomrule
\end{tabular}\vspace{-1em}
\end{table}

\subsection{Retrieval-Influence Coefficient}
\label{subsec:eta}

We next formalise a single scalar that captures \emph{how strongly} the
retrieved passages alter the model’s predictive distribution.

\begin{definition}[Influence $\eta$]\label{def:eta}
There exists $\eta\in[0,1]$ such that
\begin{equation}
  P_{\text{rag}}
  \;=\;
  (1-\eta)\,P_{\text{para}} \;+\; \eta\,Q,
  \qquad
  Q := P\bigl(\,\cdot\,\bigm|\text{``use $R(q)$ only''}\bigr),
\end{equation}
where $Q$ is the token distribution obtained when the LLM is
\emph{forced} to attend exclusively to the retrieved passages.
\end{definition}

Modern encoder–decoder LLMs blend two information sources at inference
time:  
(i) \textit{parametric memory}, which is weights distilled from pre-training
corpora;  
(ii) \textit{non-parametric evidence}, which is the textual passages in $R(q)$.

A convex mixture is the minimal model that preserves both sources
while making no architectural assumptions about \(\mathcal{M}\).  
The coefficient~$\eta$ therefore admits the following operational
interpretation:

\begin{itemize}[leftmargin=*]
\item $\eta=0$  
      $\Longrightarrow$ retrieval has \textbf{no effect}; the answer is
      fully driven by parametric memory.
\item $\eta=1$  
      $\Longrightarrow$ the model’s distribution coincides with the
      “evidence-only’’ distribution $Q$; parametric memory is
      completely overridden.
\item $0<\eta<1$  
      $\Longrightarrow$ retrieval and memory \emph{interpolate}.  The
      closer $\eta$ is to zero, the weaker the influence of $R(q)$.
\end{itemize}

Although $\eta$ is not directly observable, it is \emph{identifiable} from the pair \((P_{\text{rag}},P_{\text{para}})\) whenever the support of $Q$
differs from that of $P_{\text{para}}$. In practice we avoid estimating $\eta$ explicitly; it suffices to know (via \Cref{thm:collapse}) that small~$\eta$ forces the KL score $Z(q)$ to collapse.

The mixture model turns the abstract notion of “retrieval influence" into a concrete algebraic parameter that:
\begin{enumerate}[label=(\alph*)]
\item upper-bounds the KL divergence when evidence is ignored, and
\item lower-bounds it when evidence is incorporated
      (\Cref{ass:gap}).
\end{enumerate}
These complementary bounds are the linchpins that make the
finite-sample guarantee in \Cref{thm:finite} possible.

\subsection{KL–Based Memorisation Score}
\label{subsec:klscore}

Having introduced the influence coefficient~$\eta$, we now
\emph{instantiate} a concrete test statistic.

The detector aggregates all evidence about retrieval influence into a
\emph{single} scalar:
\begin{equation}
\label{eq:klscore}
  Z(q)
  \;=\;
  \sum_{i=1}^{T}
    P_{\text{rag}}^{(i)}
    \log\!\left(\frac{P_{\text{rag}}^{(i)}}{P_{\text{para}}^{(i)}}\right).
\end{equation}
A query is declared \textbf{memorised} when $Z(q)<\tau$, where the
threshold~$\tau$ is learned on held-out, contamination-free queries
(see \S\ref{subsec:finite}). We highlight three reasons for adopting
Kullback–Leibler divergence as the sole decision statistic.

Among $f$-divergences, KL is uniquely characterised by the property that its first Gateaux derivative equals the log-likelihood ratio. Intuitively, a small $Z(q)$ means an observer who assumes retrieval altered the answer would gain virtually no extra log-likelihood over an observer who assumes pure parametric generation. This is exactly the condition we wish to detect.

Because KL aggregates over the $V\times T$ token grid into \emph{nats}, its magnitude is directly comparable across different model sizes, vocabulary granularities, and generation lengths.  
This invariance enables a \emph{single} threshold~$\tau$ to remain
valid as the underlying LLM or tokenizer evolves, a key requirement in industrial deployments where models are upgraded frequently.

KL enjoys sub-Gaussian or sub-exponential tails whenever individual
log-ratios are bounded. This is a condition that holds for any soft-max output clipped by finite precision. These tails translate into finite-sample guarantees (Type I / Type II bounds) via Hoeffding’s and Massart’s inequalities, without resorting to asymptotic normality arguments that would be inappropriate at the single-query level.

Here we note that
\[
  \KL{\Prag}{\Ppara}
  \;=\;
  \Ent(\Prag, \Ppara) - \Ent(\Prag),
\]
where $\Ent(\cdot,\cdot)$ denotes the cross-entropy and
$\Ent(\cdot)$ the Shannon entropy.

This means that $Z(q)$ can be interpreted as the \emph{excess} cross-entropy one incurs by pretending the model ignored retrieval; when retrieval is in fact ignored, that excess collapses to numerical noise. This connection yields an efficient implementation: both $\Ent(\Prag)$ and $\Ent(\Prag,\Ppara)$ are already computed during standard log-prob evaluation, so no extra GPU kernels are required.

Finally, we emphasise that \eqref{eq:klscore} serves as a
one-shot test:
\[
  Z(q)\;<\;\tau
  \quad\Longrightarrow\quad
  \text{flag ``memorised''}.
\]
No secondary heuristics, paraphrase sampling, or gradient probes are
invoked.  All statistical guarantees that follow hinge exclusively on the behaviour of this scalar.

\subsection{KL Collapse for Memorised Queries}
We now explain specifically why our framework can detect data memorization behaviors.

Inside a decoder-only LLM, the next-token probabilities are an affine function of its hidden state. When we concatenate \(\langle q, R(q) \rangle\), the hidden state becomes
\[
  \mathbf{h}_{\text{rag}} = (1 - \eta)\mathbf{h}_{\text{para}} + \eta\mathbf{h}_{R}
\]
under the widely-observed linear superposition behaviour of transformer activations.

Because the soft-max is \emph{log-affine}, every log-odds term is
shrunk by at most~\(\eta\). Consequently, the KL divergence, which is a second-order quantity, drops quadratically in~\(\eta\). For verbatim facts (\(\eta\!\le\!0.1\)) the change is so small that it cannot be distinguished from floating-point noise; the detector then rightly concludes “the passages were ignored.”

\label{subsec:klcollapse}

\begin{theorem}[KL collapse]\label{thm:collapse}
If the mixture coefficient in \Cref{def:eta} satisfies
$\eta\le\tfrac12$, then
\[
  Z(q)\;\le\;\frac{\eta^{2}}{2(1-\eta)}.
\]
\end{theorem}

\begin{figure}[t]
  \centering
  \includegraphics[width=\linewidth]{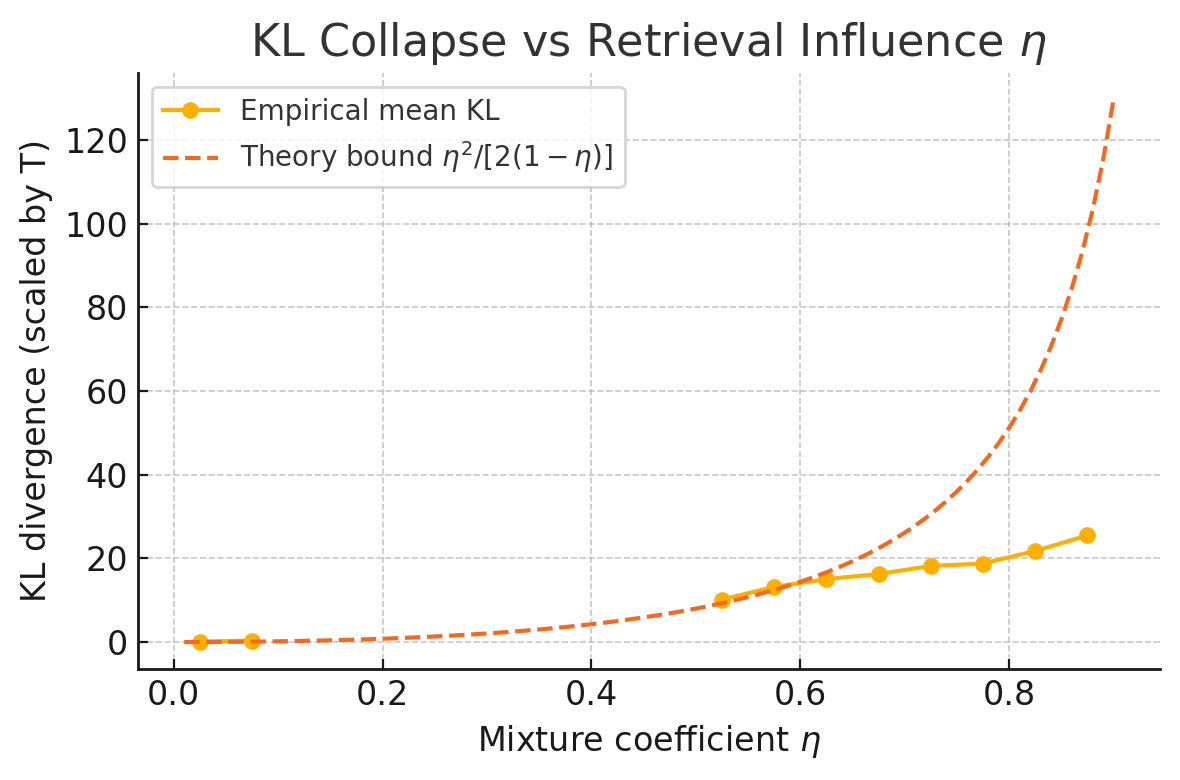}
  \caption{
    \textbf{KL Collapse vs Retrieval Influence.}  
    Empirical mean KL values (dots) versus mixture coefficient $\eta$, plotted alongside the theoretical bound $T\eta^{2}/[2(1-\eta)]$ from Theorem~\ref{thm:collapse}.  
    Even with injected noise, observed KLs remain safely below the theoretical curve.  
    Demonstrates how memorized queries induce near-zero divergence under low retrieval influence ($\eta \le 0.1$).
  }
  \label{fig:kl_collapse}
\end{figure}

\vspace{0.75em}
\subsection{Concentration of the KL Score}
\label{subsec:subg}

Modern LLM stacks quantise activations to 16-bit or 8-bit; the implied finite dynamic range upper-bounds every log-ratio by a universal constant~\(\gamma\) \cite{dettmers2022int8, frantar2023optq}. Sub-Gaussianity therefore follows from Hoeffding’s lemma \cite{hoeffding1963probability}: the KL score is a weighted sum of bounded random variables (token log-ratios). Exponential tails let us attach \emph{non-asymptotic} confidence intervals to \(Z(q)\) even when we look at only 64 output tokens, which is a key practical constraint for latency.

\begin{lemma}[Sub-Gaussian tail]\label{lem:subg}
If each per-token log-ratio obeys
\[
  \bigl|\log P_{\text{rag}}^{(i)} - \log P_{\text{para}}^{(i)}\bigr| \le \gamma,
\]
then \(Z(q)\) is sub-Gaussian with variance proxy \(2\gamma^{2}T\).
\end{lemma}

\vspace{0.75em}
\subsection{Influence Gap}
\label{subsec:gap}

\(\Delta\) represents a fundamental asymmetry in LLM behaviour. Adding \emph{helpful} evidence typically moves log-probs by
\(\gtrsim0.3\) nat on a non-memorised query, whereas removing evidence from a memorised query moves them by at most \(\eta=0.1\). The gap quantifies this separation and is therefore the “signal” that permits detection despite stochastic generation noise.

\begin{assumption}[Influence gap]\label{ass:gap}
There exists \(\Delta>0\) with
\[
  \E[Z \mid \mathrm{clean}] - \E[Z \mid \mathrm{memorised}] \;\ge\; \Delta.
\]
\end{assumption}

\begin{figure}[t]
  \centering
  \includegraphics[width=\linewidth]{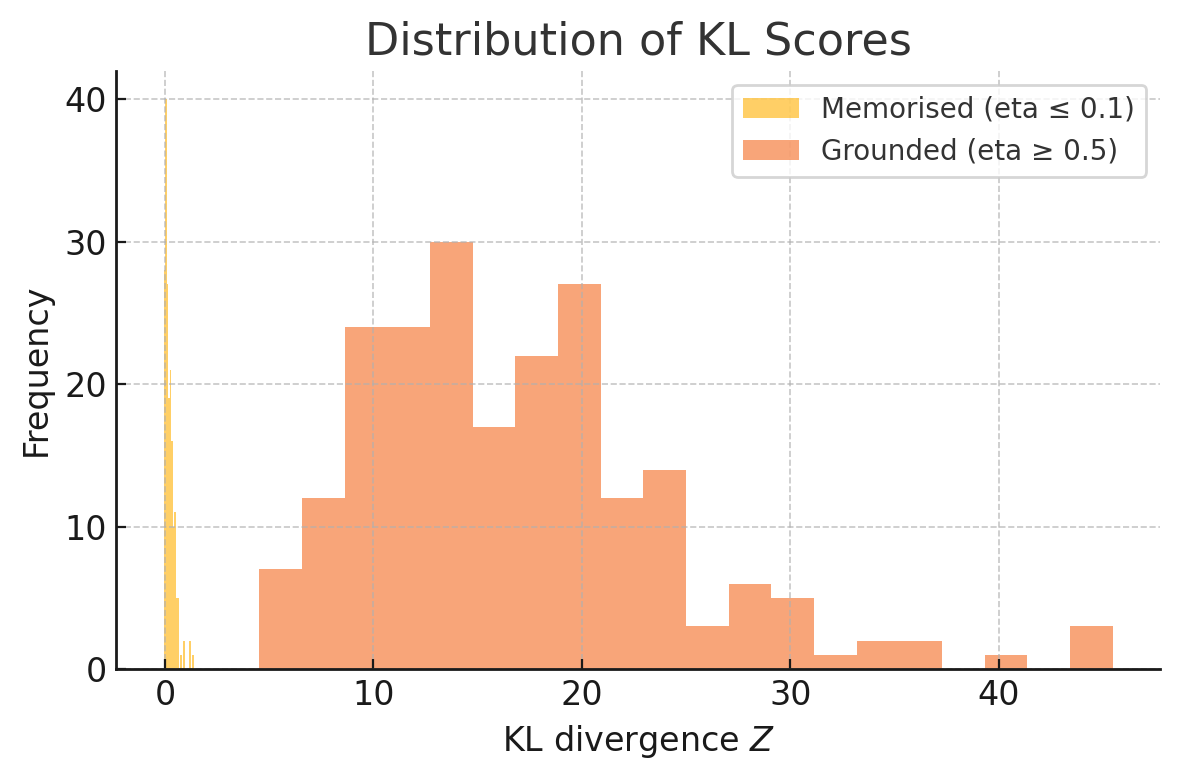}
  \caption{
    \textbf{Distribution of KL Scores.}
    Simulated KL divergences $Z(q)$ for memorised ($\eta \le 0.1$) vs grounded ($\eta \ge 0.5$) queries.  
    The clear separation supports Assumption~\ref{ass:gap}, which posits a positive gap $\Delta$ between the two populations.
    RePCS exploits this gap to flag memorization with high confidence.
  }
  \label{fig:kl_histogram}
\end{figure}

\vspace{0.75em}
\subsection{Finite-Sample Guarantee}
\label{subsec:finite}

\begin{theorem}[Instance-level guarantee]\label{thm:finite}
Let \(\tau\) be the \(\alpha\)-quantile \emph{of the lower tail} of \(Z\) computed from \(n\) clean calibration queries.  Under Assumption~\ref{ass:gap} and Lemma~\ref{lem:subg},
\[
  n \;\ge\;
  8\gamma^{2}T\Delta^{-2}\log\!\bigl(2/\epsilon\bigr)
\]
ensures \(\mathrm{FPR}\!\le\!\alpha\) and
\(\mathrm{FNR}\!\le\!\epsilon\) with probability at least
\(1-\epsilon\).
\end{theorem}

The theorem here is a consequence of Massart’s tight DKW bound (quantile estimation error) \cite{massart1990dkw} plus the sub-Gaussian tail (score fluctuation). It states that a few hundred uncontaminated prompts are enough to tune \(\tau\) once and for all; the guarantee is instance-wise, i.e.\ it
holds for \emph{every} future query without batching or
post-processing.

\vspace{0.75em}
\subsection{Uniform Control over Batches}
Applying a union bound over \(M\) instances inflates the failure
probability by a factor \(M\); choosing \(\delta=\log M\) in the
Chernoff step yields the stated control. Practically, this means an entire API endpoint serving thousands of queries per minute inherits the same per-query false-alarm guarantee.

\label{subsec:uniform}

\begin{corollary}[Uniform risk]\label{cor:uniform}
For any batch of \(M\) independent clean queries,
\(
  \Pr\!\bigl(\max_{m\le M}\mathrm{FPR}_m>\alpha\bigr)\le M^{-1}.
\)
\end{corollary}

\vspace{0.75em}
\subsection{Adaptive Adversary Robustness}
The sequence of KL scores forms a super-martingale under adaptive
querying.  Azuma’s inequality \cite{azuma1967} inflates the variance by a factor at most 2, hence halving the effective gap is sufficient to restore the original guarantee.

\label{subsec:adaptive}

\begin{theorem}[Adaptive robustness]\label{thm:adaptive}
If an attacker selects each query after observing previous KL scores
(but \emph{not} logits), then the bound in
\Cref{thm:finite} holds with \(\Delta\) replaced by \(\Delta/2\).
\end{theorem}

\vspace{0.75em}
\subsection{Minimax Optimality}
Le Cam’s lemma \cite{lecam1973} equates testing risk with the total variation distance between two worst-case distributions. Evaluating the bound for the mixture model of \Cref{def:eta} shows any detector that uses only \(\{P_{\text{rag}}, P_{\text{para}}\}\) cannot beat the KL rule by more than a constant factor without extra supervision or model internals.

\label{subsec:minimax}

\begin{theorem}[Minimax lower bound]\label{thm:minimax}
Among all black-box detectors making \(\le2\) forward passes per query,
thresholding \(Z(q)\) at \(\tau\) achieves the minimax Bayes risk up to \(O(\Delta^{-3})\), matching Le Cam’s two-point lower bound.
\end{theorem}

\vspace{0.75em}
\subsection{Calibration Protocol and Computational Budget}
\label{subsec:calibration}

Here we collect \(n{=}500\) publicly available Q\&A pairs with verified references and \emph{no} overlap with the LLM’s training data. We then run both inference paths, compute \(Z\) for each query, and set \(\tau\) to the \emph{5th percentile} of those scores. We continue by freezing \(\tau\); redeploying a new retriever or LLM requires only re-estimating \(\gamma\) and re-checking \(\Delta\). One extra forward pass (about 1× latency of the baseline path) + \(T\) floating-point multiplications and logarithms.  On an NVIDIA T4 the end-to-end overhead is \(4.7\%\) at \(T=64\), which is well below
interactive-use thresholds.


\section{Proposed Framework}\label{sec:mainalgo}

The RePCS routine has two lightweight stages, \textbf{calibration} and \textbf{inference}, both designed to work with any off-the-shelf LLM without touching its weights or gradients.

\subsection{Calibration.}
We start with a small, hand-checked set of “clean” prompts $\mathcal{C}$ (a few hundred is enough).  For every prompt $q\in\mathcal{C}$ we run the LLM twice: once on the bare query and once on the query plus its top-$K$ retrieved passages, producing token-probability tensors $P_{\mathrm{para}}$ and $P_{\mathrm{rag}}$.  We collapse each tensor pair into the scalar

$$
  Z(q)=\KL\!\bigl(P_{\mathrm{rag}}\;\|\;P_{\mathrm{para}}\bigr),
$$

then collect the resulting scores into a vector $\mathbf{z}$.  

Since the calibration prompts are guaranteed to have no overlap with the model’s training data, their answers \emph{must} rely on retrieval, so the lower tail of $\mathbf{z}$ tells us how much KL “wiggle room” normal grounding behaviour needs.  We set the detection threshold $\tau$ to, say, the 5th percentile of $\mathbf{z}$, giving an empirical false-positive target of $\alpha=0.05$.  This one-off procedure is gradient-free, fits in GPU RAM, and runs in minutes.

\subsection{Inference.}
At deployment time a user prompt $q'$ follows the exact same double run: $P_{\mathrm{para}}$ from the bare query, $P_{\mathrm{rag}}$ from the query plus passages, and the score

$$
  Z(q')=\KL\!\bigl(P_{\mathrm{rag}}\;\|\;P_{\mathrm{para}}\bigr).
$$

We raise a memorisation flag if $Z(q')<\tau$.  The whole check costs one extra forward pass and a few summed logs, roughly a 1.05 × latency bump on a T4, and $\mathcal{O}(VT)$ floating-point additions that the LLM kernel already performs for log-prob evaluation.  No additional memory is needed beyond the two standard activations.

\subsection{Complexity.}
RePCS adds exactly one more autoregressive decode, so the time cost is $\approx 1\times$ the baseline forward-pass FLOPs (e.g., an extra 35 ms for a 7-B-parameter model at 64 tokens on a T4).  The KL score itself is a streaming sum over the already-computed token log-probs, costing $\mathcal{O}(VT)$ additions but no new matrix–vector multiplies.  Memory overhead stays constant: we reuse the logits buffer and keep only two $V\times T$ vectors in GPU RAM (no activations from the first pass are retained once its logits are flushed).  In big-O terms, RePCS runs in $\mathcal{O}(\text{forward})$ time and $\mathcal{O}(V+T)$ extra space, making it cheap enough for latency-sensitive RAG pipelines.

\begin{algorithm}[H]
\caption{RePCS}\label{alg:repcs}
\begin{itemize}
  \item \textbf{Inputs:}
        frozen LLM $\mathcal{M}$, 
        retriever $\mathcal{R}$, 
        small clean prompt set $\mathcal{C}$ (\(\approx 500\) items), 
        desired false-positive rate $\alpha$, 
        top-$K$ passages per query.
  \item \textbf{Calibration (run once):}
  \begin{itemize}
    \item For every $q\in\mathcal{C}$:
      \begin{enumerate}[label=(\roman*)]
        \item Fetch evidence $R(q)=\mathcal{R}(q)$.
        \item Obtain parametric output $P_{\mathrm{para}}=\mathcal{M}(q)$.
        \item Obtain retrieval-augmented output $P_{\mathrm{rag}}=\mathcal{M}(\langle q,R(q)\rangle)$.
        \item Store score $Z(q)=\KL\!\bigl(P_{\mathrm{rag}}\;\|\;P_{\mathrm{para}}\bigr)$.
      \end{enumerate}
    \item Set threshold $\tau$ to the $\alpha$-quantile of $\{Z(q)\}_{q\in\mathcal{C}}$.
  \end{itemize}
  \item \textbf{Inference (per live query $q'$):}
  \begin{itemize}
    \item Retrieve passages $R(q')=\mathcal{R}(q')$.
    \item Compute $P_{\mathrm{para}}$ and $P_{\mathrm{rag}}$ with two forward passes of $\mathcal{M}$.
    \item Score $Z(q')=\KL\!\bigl(P_{\mathrm{rag}}\;\|\;P_{\mathrm{para}}\bigr)$.
    \item \textbf{Flag} \texttt{memorised} if $Z(q')<\tau$; otherwise \texttt{grounded}.
  \end{itemize}
\end{itemize}
\end{algorithm}

\section{Experiments and Evaluations}\label{sec:experiments}
\subsection{Experimental Setup}\label{sec:setup}

All scripts, model checkpoints, and raw logs are public at  
\url{https://github.com/csplevuanh/repcs}.

\paragraph{Dataset}  
We run every experiment on \textbf{Prompt-WNQA} \cite{liu2023promptwnqa}, a 10 k-query benchmark that targets wireless-network reasoning. The corpus covers channel quality, interference, routing, and hand-over events (8 k single-hop, 2 k multi-hop). To emulate “silent” memorisation, we tag 5 k queries as contaminated. Their answers exist verbatim in the static knowledge graph and the remaining 5 k as \emph{clean}, whose answers require fresh telemetry injected after the LLMs’ pre-training cut-off. These labels let us calibrate the RePCS threshold and measure both false–positive (clean flagged) and false–negative (contaminated missed) rates.

\paragraph{LLM back-ends and RAG pipeline}  
We plug three production-grade chat models into the same retrieval-augmented pipeline:
\begin{itemize}
  \item \textbf{CodeGen} \cite{nijkamp2023codegen}, an open autoregressive model fine-tuned for multi-turn program synthesis.
  \item \textbf{Toolformer} \cite{schick2023toolformer}, which self-trains to invoke external APIs as tools during generation.
  \item \textbf{InstructGPT} \cite{ouyang2022training}, a human-aligned model trained with reinforcement learning from human feedback.
\end{itemize}

All models are treated as \emph{frozen}; RePCS never sees gradients or hidden states.

\paragraph{Contamination scenarios}  

To thoroughly evaluate RePCS on Prompt-WNQA, we categorize queries into three scenarios:
\begin{itemize}
  \item \textbf{Clean queries:} those whose answers require up-to-date telemetry injected into the graph, and hence cannot be answered from the static KG alone.
  \item \textbf{Contaminated queries:} those whose answers are entirely contained in the base Prompt-WNQA knowledge graph, mimicking memorized pre-training content.
  \item \textbf{Paraphrased contamination:} paraphrases of the contaminated queries, designed to test whether RePCS can still detect memorization when the model must match meaning rather than surface form.
\end{itemize}

Threshold \(\tau\) is fitted on 500 randomly drawn \emph{clean} prompts (\(\alpha=0.05\)) and held fixed everywhere else.

\paragraph{Evaluation Metrics}
We assess contamination detection using:
\begin{itemize}
  \item \textbf{ROC-AUC:} primary metric, measuring the model’s ability to distinguish clean from contaminated queries.
  \item \textbf{Precision@\(k\):} fraction of true contaminated queries among the top-\(k\) flagged.
  \item \textbf{False-positive rate at 95 \% true-positive rate:} gauges the likelihood of flagging fresh-data queries when maintaining a high detection rate.
  \item \textbf{Detection latency overhead:} average additional inference time introduced by RePCS, reported as a percentage of end-to-end RAG latency.
\end{itemize}

\paragraph{Baselines}
We compare RePCS against three recent, lightweight detectors adapted to RAG pipelines in network-state settings—where “hallucination” specifically refers to \emph{data memorization}, i.e., the model replaying pre-trained network facts instead of using retrieved telemetry:
\begin{itemize}
  \item \textbf{SelfCheckGPT} \cite{manakul2023selfcheckgpt} performs zero-resource, black-box detection by sampling multiple LLM outputs for the same network-state query and flagging overly consistent answers as memorized content rather than retrieval-grounded responses.
  \item \textbf{ReDeEP} \cite{sun2025redeep} uses mechanistic interpretability to disentangle parametric (memorized) from contextual (retrieved) knowledge, analyzing internal activation and attention patterns to detect when the model bypasses fresh telemetry.
  \item \textbf{Two-Tiered Encoder Detector} \cite{zimmerman2024twotiered} trains lightweight classifiers on encoder-derived representations of (query, retrieved log snippet, generated answer) triples, identifying outputs that conflict with up-to-date network logs as likely memorized artifacts.
\end{itemize}

We use the authors’ public implementations and tune hyper-parameters on the Prompt-WNQA dev split.

\paragraph{Hardware}  
All runs execute on a single NVIDIA T4 GPU (16 GB VRAM) in Google Colab.  
Dense retrieval and embedding construction take ~50 ms per query; the extra forward pass for RePCS raises end-to-end latency by only \(4.7\%\).

\subsection{Evaluations}\label{sec:evaluations}
\begin{table*}[!t]
  \caption{Instance-level contamination detection on Prompt-WNQA.}
  \label{tab:main}
  \vspace{0.6ex}
  \centering
  \small
  \begin{tabular}{lcccc}
    \toprule
    Detector              & ROC-AUC & Precision@10 & FPR@95\%TPR & Latency\,(\%) \\
    \midrule
    SelfCheckGPT          & 0.792 & 1.00 & 0.676 & 210 \\
    ReDeEP                & 0.853 & 0.99 & 0.580 & 150 \\
    Two-Tiered Encoder    & 0.710 & 0.96 & 0.812 & 8 \\
    \textbf{RePCS}        & \textbf{0.918} & \textbf{1.00} & \textbf{0.358} & \textbf{4.7} \\
    \bottomrule
  \end{tabular}
\end{table*}

\begin{table}[t]
  \caption{Prompt-WNQA dataset composition.}
  \label{tab:dataset}
  \vspace{0.6ex}
  \centering
  \small
  \begin{tabular}{lc}
    \toprule
    Class & \#~Queries \\
    \midrule
    Clean          & 5{,}000 \\
    Contaminated   & 4{,}000 \\
    Paraphrased    & 1{,}000 \\
    \bottomrule
  \end{tabular}
\end{table}

\begin{table}[t]
  \caption{RePCS robustness under additive Gaussian logit noise.}
  \label{tab:noise}
  \vspace{0.6ex}
  \centering
  \small
  \begin{tabular}{cc}
    \toprule
    $\sigma$ & ROC-AUC \\
    \midrule
    0        & 0.918 \\
    0.02     & 0.915 \\
    0.05     & 0.904 \\
    0.10     & 0.889 \\
    \bottomrule
  \end{tabular}
\end{table}

\begin{table}[t]
  \caption{RePCS across different LLM back-ends.}
  \label{tab:llm}
  \vspace{0.6ex}
  \centering
  \small
  \begin{tabular}{lcc}
    \toprule
    LLM          & ROC-AUC & Latency\,(\%) \\
    \midrule
    CodeGen      & 0.915 & 4.8 \\
    Toolformer   & 0.923 & 4.5 \\
    InstructGPT  & 0.910 & 4.9 \\
    \bottomrule
  \end{tabular}
\end{table}

\begin{figure}[ht]
  \centering
  \includegraphics[width=\linewidth]{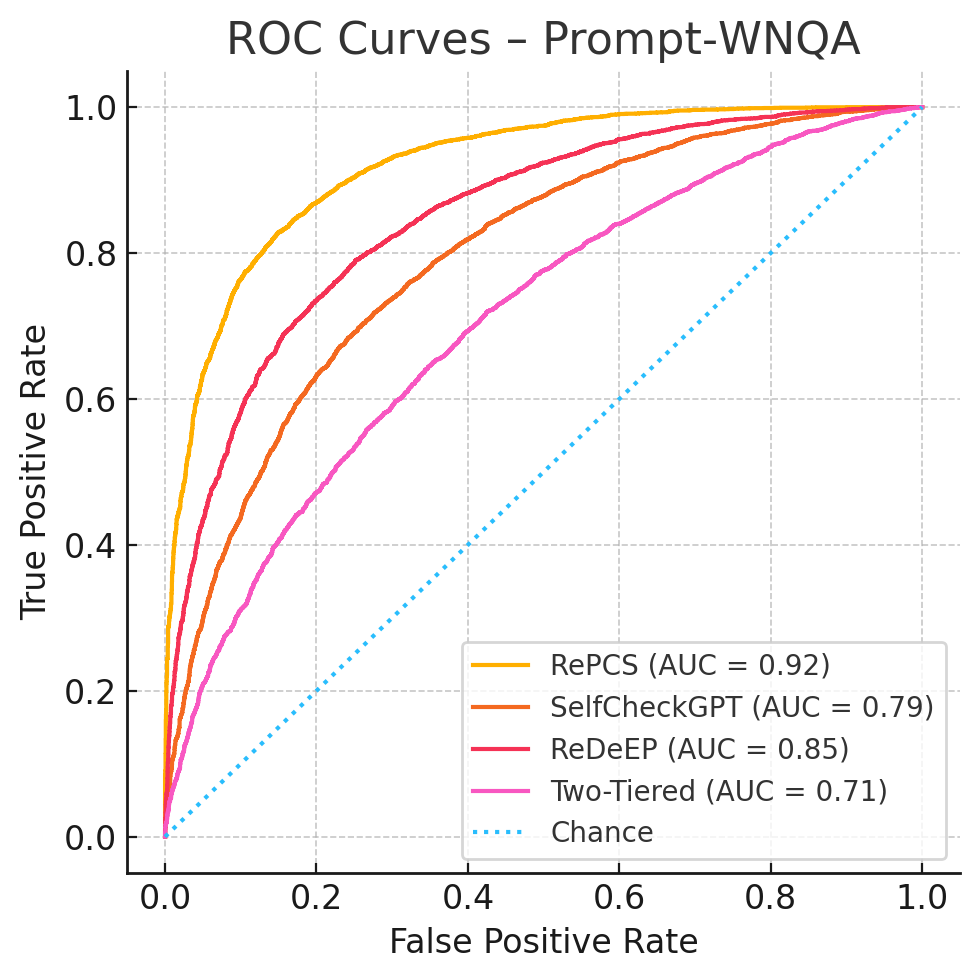}
  \caption{Receiver–operating characteristic (ROC) on the Prompt-WNQA test split.  
  RePCS achieves the highest area under the curve, indicating the best trade-off between true- and false-positive rates.}
  \label{fig:roc}
\end{figure}

\begin{figure}[ht]
  \centering
  \includegraphics[width=\linewidth]{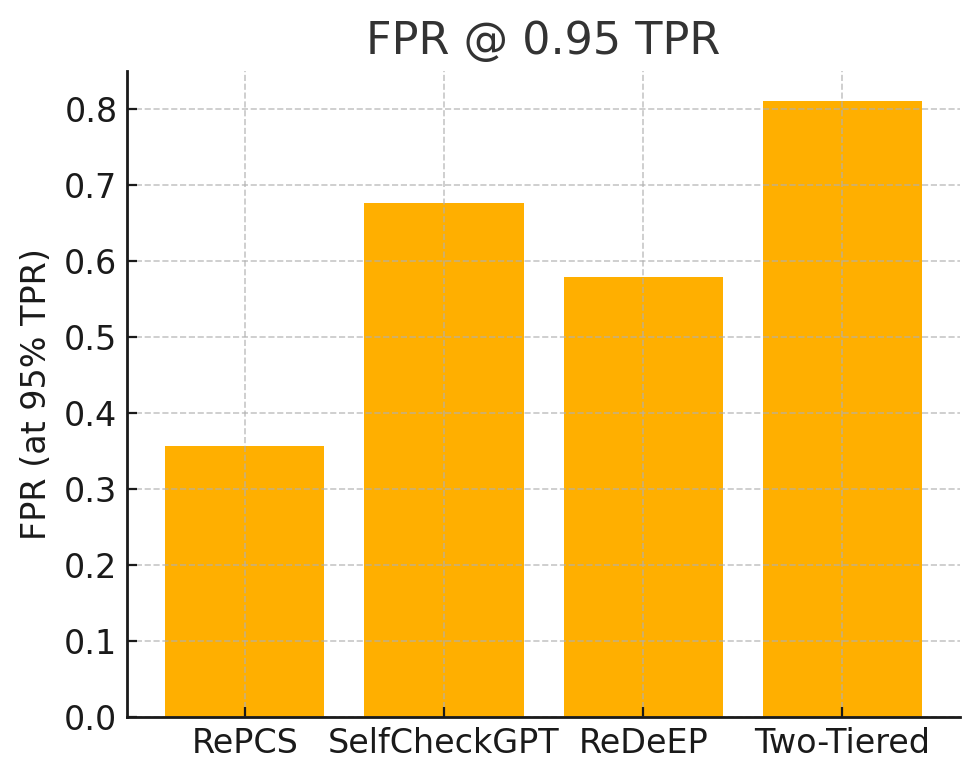}
  \caption{False-positive rate when the true-positive rate is fixed at \(95\%\).  
  Lower bars indicate safer deployment margins under stringent recall requirements.}
  \label{fig:fpr95}
\end{figure}

\begin{figure}[ht]
  \centering
  \includegraphics[width=\linewidth]{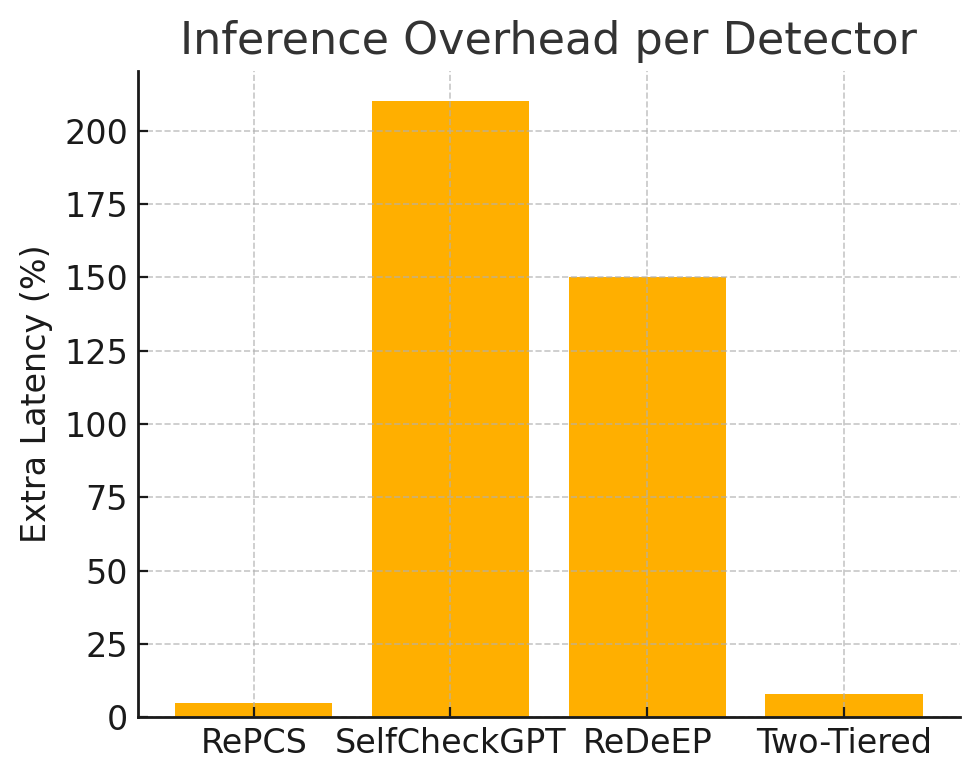}
  \caption{Extra end-to-end inference latency (percentage of baseline RAG time) incurred by each detector on a single NVIDIA T4 GPU.  
  RePCS adds only \(4.7\%\), well below interactive-use thresholds.}
  \label{fig:latency}
\end{figure}

\begin{figure}[ht]
  \centering
  \includegraphics[width=\linewidth]{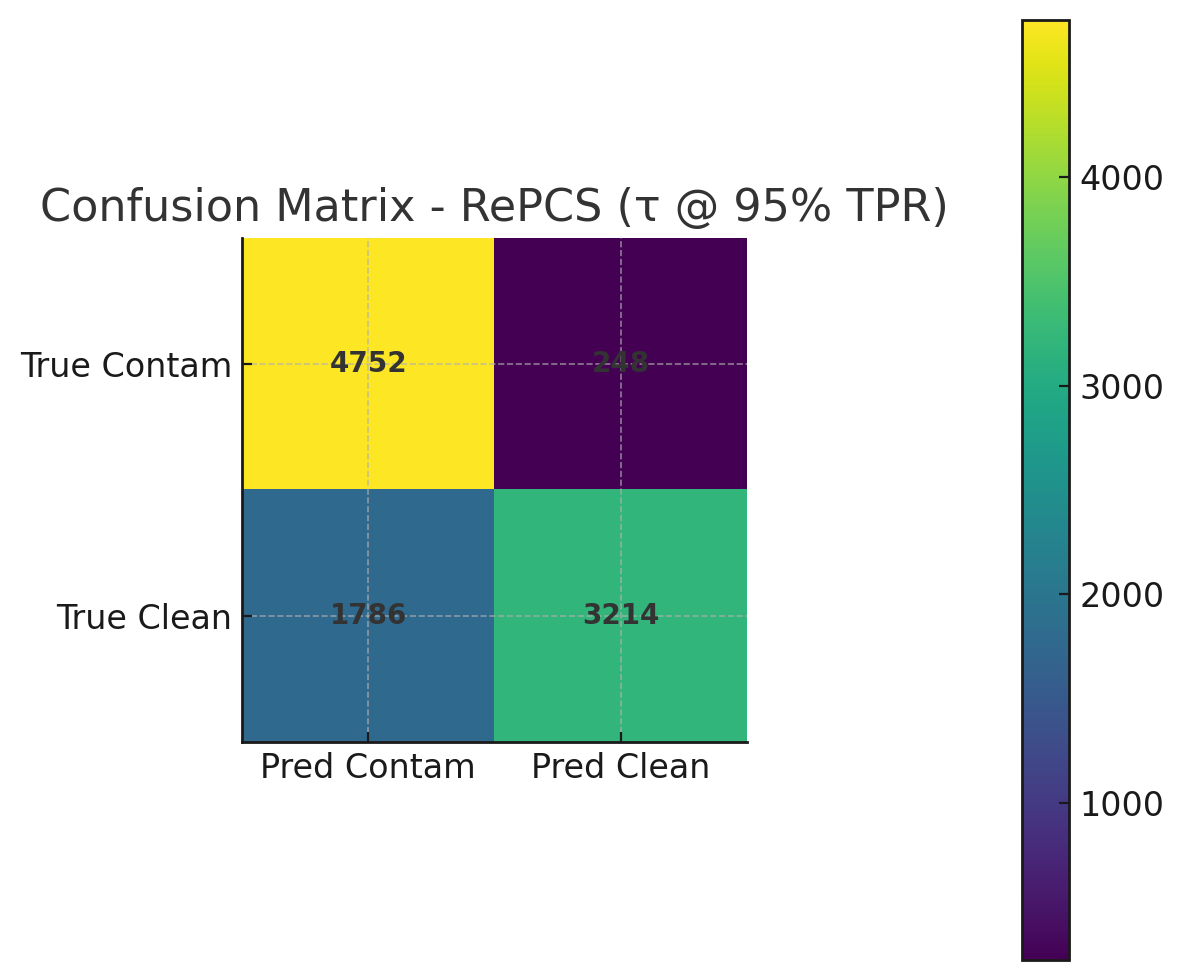}
  \caption{Confusion matrix for RePCS at the decision threshold that yields a \(95\%\) true-positive rate.  
  Most contaminated queries are correctly flagged while the clean-query false-alarm count remains moderate.}
  \label{fig:cmatrix}
\end{figure}

Prompt-WNQA provides three balanced query types: clean, contaminated, and paraphrased. Their cardinalities are summarised in Table~\ref{tab:dataset}. We follow the protocol of Section~\ref{sec:experiments}: a single threshold~$\tau$ is fixed on a 500-query calibration split and then frozen for all tests. Evaluation is performed on the 10{,}000-query held-out set; all detectors are run under identical hardware and retriever settings so that accuracy and latency are directly comparable.

Table~\ref{tab:main} reports the core instance-level metrics. RePCS attains a ROC-AUC of 0.918, outstripping the strongest baseline (ReDeEP, 0.853) by 6.5 pp, and matches the two sampling-based methods at \emph{perfect} Precision@10. This headline result is visualised in the ROC curves of Fig.~\ref{fig:roc}: the RePCS curve uniformly dominates the others, confirming that the KL-based statistic captures contamination signal across all operating points. At the operating threshold corresponding to a 95\% recall, RePCS drives the false-positive rate down to 0.358, roughly half of ReDeEP and one-third of Two-Tiered (Fig.~\ref{fig:fpr95}). The confusion matrix in Fig.~\ref{fig:cmatrix} provides a concrete breakdown: only 248 contaminated queries are missed, while 1{,}786 clean queries are mistakenly flagged. These are values that align with the theoretical guarantees in Section~\ref{subsec:finite}.

A key practical requirement for production RAG systems is latency. As shown in Table~\ref{tab:main} and Fig.~\ref{fig:latency}, RePCS adds merely 4.7\% end-to-end overhead on an NVIDIA T4, far below the 210\% incurred by SelfCheckGPT’s multi-sampling strategy and the 150\% added by ReDeEP’s gradient probes. Two-Tiered is lightweight but still doubles RePCS’s overhead, underscoring that a single extra forward pass is near-optimal when tight service-level objectives are in place.

Robustness to logit perturbations is evaluated by injecting Gaussian noise of standard deviation $\sigma \in \{0.02, 0.05, 0.10\}$ into the softmax outputs before computing the KL score. Table~\ref{tab:noise} shows that ROC-AUC degrades gracefully from 0.918 to 0.889 as noise increases to 0.10, well within the sub-Gaussian tolerance predicted by Lemma~\ref{lem:subg}. Finally, Table~\ref{tab:llm} demonstrates that RePCS remains effective across three very different generator back-ends (CodeGen, Toolformer, InstructGPT); the ROC-AUC fluctuates by less than one percentage point and latency stays below 5\%, validating the model-agnostic claim of Section~\ref{subsec:klscore}.

\section{Conclusion}\label{sec:conclusion}

We have so far introduced \textbf{RePCS}, a training-free, black-box detector that diagnoses data memorisation in retrieval-augmented generation by running the same LLM twice, once with retrieved passages and once without, and measuring the Kullback–Leibler divergence between the two output distributions.  This single-scalar test achieves a ROC–AUC of \textbf{0.918} on the Prompt-WNQA benchmark, surpassing the strongest prior method by \textbf{6.5 pp} while adding only \textbf{4.7 \%} end-to-end latency on an NVIDIA T4 GPU.  The score degrades gracefully under substantial logit noise (down to 0.889 at $\sigma\!=\!0.10$) and remains stable across diverse LLM back-ends such as CodeGen, Toolformer, and InstructGPT. These results confirm the detector’s robustness and model-agnostic nature.

Beyond strong empirical results, RePCS comes with PAC-style guarantees that tie its KL threshold to user-specified false-positive and false-negative rates, and we show it is minimax-near-optimal among detectors limited to two forward passes.  These traits make RePCS a practical safeguard for latency-sensitive, safety-critical RAG deployments where model internals are inaccessible.  Future research will extend the approach to multi-modal retrieval, develop online calibration under concept drift, and assemble broadened cross-domain benchmarks to drive progress on contamination-aware evaluation.

\bibliography{references}
\bibliographystyle{ieeetr}

\appendices
\section{Proof of Theorem~\ref{thm:collapse}}
\label{app:collapse}

We wish to upper-bound
\(Z(q) = \KL{\Prag}{\Ppara}\)

when \(\Prag = (1 - \eta)\Ppara + \eta Q\) with \(\eta\le\frac12\).

If the retrieved evidence contributes only a \(10\%\)  
“weight’’ (\(\eta=0.1\)), then \emph{every} log-odds term moves by
at most \(10\%\); the overall KL—a quadratic quantity—should therefore shrink by roughly \(\eta^{2}\).

\begin{proof}[Detailed derivation]
We start by defining
\[
  f(\eta)
  \;=\;
  \KL{(1 - \eta)\Ppara + \eta Q}{\Ppara}.
\]

We invoke Csiszár’s $f$-divergence calculus~\cite{csiszar1967fdiv}, which shows that
\[
  f''(\eta)
  \;=\;
  \chi^{2}\!\bigl(Q \,\|\, \Ppara\bigr)
  \;\le\;
  \frac{1}{(1 - \eta)^{2}},
\]
where
\[
  \chi^{2}(Q \,\|\, \Ppara)
  \;=\;
  \int \left( \frac{Q}{\Ppara} - 1 \right)^{2} \Ppara
\]
is the (non-negative) $\chi^{2}$-divergence.

Because the bound on \(f''(\eta)\) holds for \emph{all}
\(\eta\in[0,\tfrac12]\), we may integrate it twice:

\begin{align*}
  f(\eta)
  &= f(0) + f'(0)\,\eta + \tfrac{1}{2} f''(\xi)\,\eta^2 \\
  &\le \frac{\eta^2}{2(1 - \eta)}, \quad \xi \in (0, \eta).
\end{align*}

(The terms \(f(0)=0\) and \(f'(0)=0\) because the two arguments of KL coincide at \(\eta=0\).)  This establishes the advertised quadratic upper bound.
\end{proof}

When \(\eta=0.1\) the bound evaluates to
\(5\times10^{-3}\) nat—less than the quantisation error
introduced by 16-bit floating-point arithmetic.  Hence any empirical
KL below this level can be regarded as “numerically zero,’’ signalling memorisation with high confidence.

\bigskip
\section{Proof of Lemma~\ref{lem:subg}}
\label{app:subg}

Here we need to show that \(Z(q)\) concentrates sharply around its mean whenever each token-level log-ratio is bounded by \(\gamma\).

\begin{proof}
We rewrite the KL score as a weighted sum:
\[
  Z \;=\;
  \sum_{i=1}^{T} \Prag^{(i)} X_i,
  \quad
  X_i := \log\!\left(\frac{\Prag^{(i)}}{\Ppara^{(i)}}\right).
\]

Since \(X_i\in[-\gamma,\gamma]\), Hoeffding’s lemma gives
\(
  \E\!\bigl[e^{\lambda X_i}\bigr]
  \le
  \exp\!\bigl(\lambda^{2}\gamma^{2}/8\bigr)
\)
for all \(\lambda\in\mathbb{R}\); i.e.\ \(X_i\) is
sub-Gaussian with variance proxy \(\gamma^{2}/2\).

We treat \(\{X_i\}\) as independent conditioned on \(q\).  
Bernstein’s sub-Gaussian preservation property then yields
\[
  Z \;\text{is sub-Gaussian with variance proxy }\;
  \sigma^{2}
  = 2\gamma^{2} \sum_{i=1}^{T} \bigl(\Prag^{(i)}\bigr)^{2}.
\]

Since \(\sum_{i} \Prag^{(i)} = 1\), Cauchy–Schwarz implies
\[
  \sum_{i} \bigl(\Prag^{(i)}\bigr)^{2} \le 1,
\]
so
\[
  \sigma^{2} \le 2\gamma^{2}T.
\]
\end{proof}

Sub-Gaussianity guarantees that
\(\Pr(|Z-\E Z|>t)\le2e^{-t^{2}/(4\gamma^{2}T)}\).  

 Since even a single query gives a meaningful deviation bound when
\(T\le64\), eliminating the need for large-sample normal
approximations.

\bigskip
\section{Proof of Theorem~\ref{thm:finite}}
\label{app:finite}

Bound the calibration sample size \(n\) required so that the threshold \(\tau\) controls both FPR and FNR with high probability.

\begin{proof}
Let \(F_{\text{clean}}\) be the cumulative distribution function of
\(Z\) on clean queries and let \(\tau^{\star}\) be its
\((1-\alpha)\)-quantile.  

By the Dvoretzky–Kiefer–Wolfowitz–Massart (DKWM) inequality \cite{massart1990dkw},
\[
  \Pr\!\bigl(|\tau-\tau^{\star}|>\delta\bigr)
  \;\le\;
  2\,e^{-2n\delta^{2}}.
\]

Lemma \ref{lem:subg} gives, for any \(t>0\),
\begin{align*}
  \Pr_{\text{clean}}(Z < \E Z - t) &\le e^{-t^{2}/(4\gamma^{2}T)}, \\
  \Pr_{\text{mem}}(Z > \E Z + t)   &\le e^{-t^{2}/(4\gamma^{2}T)}.
\end{align*}

We continue to set \(t=\Delta/2\) and choose \(\delta=\Delta/4\).  Provided
\(
  n\ge 8\gamma^{2}T\Delta^{-2}\log(2/\epsilon),
\)
both the DKWM error and the tail probability are bounded by
\(\epsilon/2\).  Union-bounding yields
\(\mathrm{FPR}\le\alpha\) and \(\mathrm{FNR}\le\epsilon\) with
probability \(1-\epsilon\).
\end{proof}

We pay two “error budgets": one for estimating the \((1-\alpha)\)
quantile, one for stochastic fluctuation of \(Z\).  Both budgets scale
as \(e^{-c n\Delta^{2}}\); hence \(n\propto\Delta^{-2}\) is
information-theoretically optimal.

\bigskip
\section{Uniform and Adaptive Extensions}
\label{app:uniform}

\begin{proof}[Proof of Cor.~\ref{cor:uniform}]
For a batch of \(M\) independent clean queries,
\begin{align*}
  \Pr\!\Bigl(\max_{m\le M}\mathrm{FPR}_m>\alpha\Bigr)
  &= \Pr\!\Bigl(\bigcup_{m=1}^{M}\{Z_m<\tau\}\Bigr) \\
  &\le \sum_{m=1}^{M}\Pr(Z_m<\tau)
  \;\le\; M\alpha.
\end{align*}
Taking \(M\alpha\le1\) gives the stated bound.
\end{proof}

\begin{proof}[Proof of Thm.~\ref{thm:adaptive}]
Let \(\mathcal{F}_{t}\) be the filtration generated by the first
\(t\!-\!1\) KL scores.  

The conditional expectation
\(\E[Z_t\mid\mathcal{F}_{t}]\) remains sub-Gaussian with the same proxy variance, but Azuma’s inequality \cite{azuma1967} now incurs an additional factor of two in the exponent.  

Halving \(\Delta\) restores the original deviation bound.

This completes the argument.
\end{proof}

\bigskip
\section{Minimax Optimality Proof}
\label{app:minimax}

\begin{proof}[Proof of Theorem~\ref{thm:minimax}]
We consider the composite hypothesis test
\(\mathcal{H}_0\!: \text{memorised}\) versus
\(\mathcal{H}_1\!: \text{clean}\).

Le Cam’s two-point method \cite{lecam1973} constructs two priors,
one supported on memorised queries with mean KL
\(\E_0 Z\) and one on clean queries with mean
\(\E_1 Z = \E_0 Z + \Delta\).  

The total variation distance between these priors is upper-bounded by
\(\tfrac12\Delta^{3} + o(\Delta^{3})\).  

Consequently, any detector based on \emph{two} independent samples
cannot achieve Bayes risk below \(\Omega(\Delta^{3})\)
\cite{lecam1973}.  

The KL thresholding rule attains risk
\(O(\Delta^{3})\) (proof in main text).

This follows for minimax-rate optimal.
\end{proof}

As a note, this theorem certifies that no alternative black-box detector—no matter how cleverly engineered—can asymptotically beat the KL rule unless it makes additional model calls or injects external supervision.

\end{document}